\documentclass[11pt]{article}

\usepackage{amssymb,amsmath,amsthm,bbm,bm}
\usepackage{verbatim,float,url,dsfont}
\usepackage{graphicx,subcaption,psfrag}
\usepackage{algorithm}
\usepackage[noend]{algorithmic}
\usepackage{mathtools,enumitem}
\usepackage{multirow}
\usepackage{ragged2e}
\usepackage{xr-hyper}
\usepackage{caption}
\usepackage{array}

\usepackage[utf8]{inputenc} 
\usepackage[T1]{fontenc}    
\usepackage{booktabs}       
\usepackage{nicefrac}         
\usepackage{microtype}      

\usepackage[dvipsnames]{xcolor} 
\definecolor{hunyuanblue}{HTML}{1E4A8F}
\usepackage{pifont} 
\usepackage{etoc} 
\usepackage{tikz} 
\usepackage{pgfplots}  
\pgfplotsset{compat=1.16}  

\ifdefined\TimesFont 
\usepackage{times} 
\fi

\ifdefined\ParSkip 
\usepackage{parskip} 
\fi

\newtheorem*{assumption*}{\assumptionnumber}
\providecommand{\assumptionnumber}{}


\makeatletter
\newcommand*\rel@kern[1]{\kern#1\dimexpr\macc@kerna}
\newcommand*\widebar[1]{%
  \begingroup
  \def\mathaccent##1##2{%
    \rel@kern{0.8}%
    \overline{\rel@kern{-0.8}\macc@nucleus\rel@kern{0.2}}%
    \rel@kern{-0.2}%
  }%
  \macc@depth\@ne
  \let\math@bgroup\@empty \let\math@egroup\macc@set@skewchar
  \mathsurround\z@ \frozen@everymath{\mathgroup\macc@group\relax}%
  \macc@set@skewchar\relax
  \let\mathaccentV\macc@nested@a
  \macc@nested@a\relax111{#1}%
  \endgroup
}
\makeatother

\DeclareMathOperator{\sign}{sign}

\def\E{\mathbb{E}}

\def\eps{{\epsilon}}
\def\1{\bm{1}}

\usepackage{hunyuan}

\usepackage[capitalize,noabbrev]{cleveref}
\usepackage{tabularx}
\usepackage[most]{tcolorbox}
\usepackage{wrapfig}
\usepackage{xspace}
\usepackage{fancyhdr}
\usepackage{xurl}

\makeatletter
\long\def\@makecaption#1#2{%
  \vskip 10pt
  \setbox\@tempboxa\hbox{#1: #2}%
  \ifdim \wd\@tempboxa >\hsize
    \noindent #1: #2\par   
  \else
    \hbox to\hsize{\hfil\box\@tempboxa\hfil}
  \fi}
\makeatother

\hypersetup{
  colorlinks=true,
  linkcolor=hunyuanblue,
  citecolor=hunyuanblue,
  urlcolor=hunyuanblue,
}

\makeatletter
\def\section{\@startsiction{section}{1}{\z@}{-0.24in}{0.10in}
             {\large\bf\raggedright\color{hunyuanblue}}}
\def\subsection{\@startsection{subsection}{2}{\z@}{-0.20in}{0.08in}
                {\normalsize\bf\raggedright\color{hunyuanblue}}}
\makeatother

\fancypagestyle{plain}{%
  \fancyhf{}%
  \fancyfoot[C]{\thepage}%
}

\fancypagestyle{firststyle}{%
  \fancyhf{}%
  \fancyhead[L]{\raisebox{-18.5mm}[0pt][0pt]{\includegraphics[height=12mm]{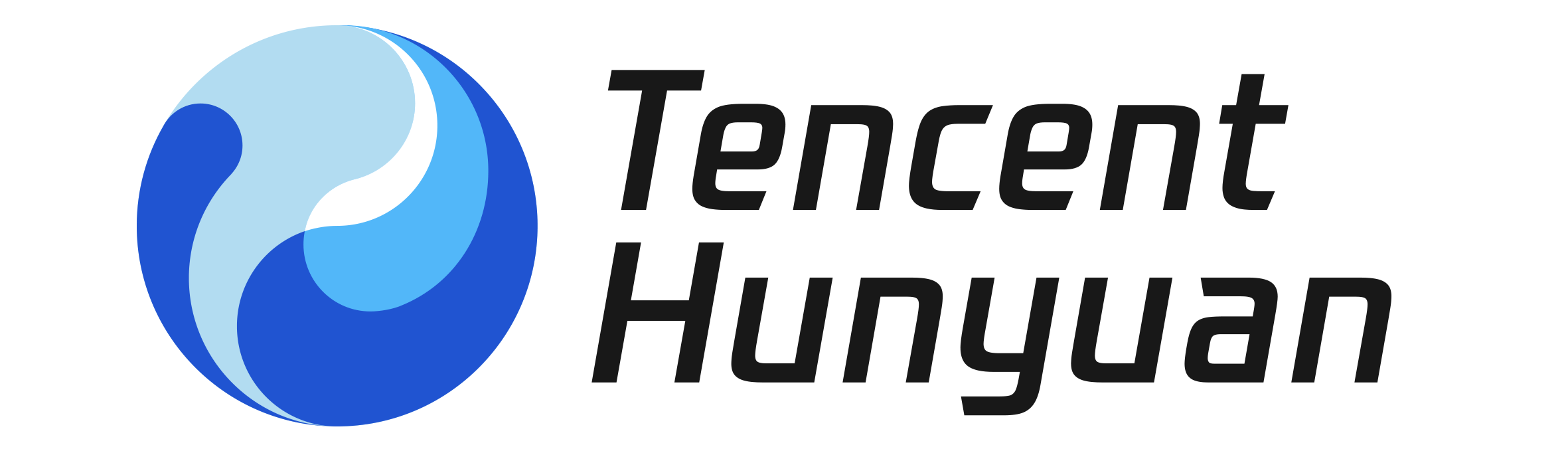}}}%
  \fancyfoot[C]{\thepage}%
}

\newcommand{\bmu}{\mu}                       
\newcommand{\tpi}{\pi}                        
\newcommand{\Adv}{\hat{A}}                    
\newcommand{\KL}{D_{\mathrm{KL}}}
\newcommand{\KLtopk}{\KL^{\mathrm{TopK}}}     
\newcommand{\Dtok}{D}                         
\newcommand{\Taylor}{\dot{D}}                 
\newcommand{\Keep}{\mathcal{K}}               

\newcolumntype{Y}{>{\raggedright\arraybackslash}X}

\definecolor{caseblue}{RGB}{42, 91, 160}
\definecolor{casebg}{RGB}{246, 248, 252}
\definecolor{caseborder}{RGB}{180, 195, 220}
\definecolor{paperblue}{HTML}{1F77B4}
\definecolor{paperred}{HTML}{D62728}
\definecolor{deepred}{HTML}{B22222}
\definecolor{softred}{HTML}{C44E52}

\newtcolorbox[auto counter, number within=section]{compactcase}[2][]{
  breakable,
  enhanced,
  colback=gray!2,
  colframe=gray!30,
  colbacktitle=gray!12,
  coltitle=black, 1
  fonttitle=\bfseries,
  title={#2},
  boxrule=0.45pt,
  arc=1mm,
  left=1.5mm,
  right=1.5mm,
  top=1mm,
  bottom=1mm,
  toptitle=0.7mm,
  bottomtitle=0.7mm,
  before skip=0.8em,
  after skip=0.8em,
  label={#1}
}

\definecolor{abstractbg}{HTML}{F0F7FC}

\begin{document}

\thispagestyle{firststyle}
\vspace*{0.25cm}
{\color{hunyuanblue}\hrule height 0.6pt}
\vskip 6mm
\begin{center}
{\LARGE\bfseries Predictive Divergence Masks for LLM RL\par}
\end{center}
\vskip 3mm
{\color{hunyuanblue}\hrule height 0.6pt}
\vskip 6mm
\begin{center}
\textbf{Xiangxin Zhou}$^{1,*}$ \quad
\textbf{Jiarui Yao}$^{1,2,*}$ \quad
\textbf{Penghui Qi}$^{3,*}$ \quad
\textbf{Bowen Ping}$^{1}$ \\
\textbf{Jiaqi Tang}$^{1}$ \quad
\textbf{Haonan Wang}$^{1}$ \quad
\textbf{Tianyu Pang}$^{1,\ddagger}$
\\[8pt]
$^1$Tencent Hunyuan \quad $^2$UIUC \quad $^3$NUS \\[6pt]
{\small $^*$Equal contribution \quad
$^\ddagger$Corresponding author}
\end{center}
\vskip 6mm

\begin{tcolorbox}[
  colframe=abstractbg,
  colback=abstractbg,
  boxrule=0pt,
  arc=2mm,
  enhanced,
  top=12pt,
  bottom=12pt,
  left=15pt,
  right=15pt,
  width=\textwidth,
]
\textbf{Abstract.}\quad
Reinforcement learning for large language models (LLMs) typically relies on
trust-region masks to stabilize off-policy updates.
The dominant PPO-style approach uses the sampled-token importance ratio for two criteria: a \emph{proximity} criterion, which asks whether the policy has moved too far from the behavior policy, and a \emph{direction} criterion, which asks whether the update pushes it farther away.
Recent work DPPO improves the proximity criterion by replacing PPO's ratio-based test with a probability divergence between the behavior and training policies.
However, its direction criterion is still inherited from PPO.
A token can be masked only when the sampled-token importance ratio moves away from one.
We observe that this ratio-based direction criterion is a single-sample proxy that can disagree in sign with the change of the divergence that defines the proximity criterion. 
We therefore propose the \emph{predictive divergence mask}, which asks whether
the next policy-gradient step will increase or decrease the same divergence used
by the trust region. For the discrete softmax policies used in LLM RL, 
we derive this prediction in closed form. Because production rollout engines expose only a truncated (top-$K$) view of the vocabulary, we develop two lightweight top-$K$ estimators for this prediction.
Detailed analysis shows the
divergence-based direction is better aligned with the realized change of the divergence
than the sampled ratio, and the resulting masks improve RL 
training across model scales and precision settings.

\vskip 8pt
\textbf{Date:} July 7, 2026 
\end{tcolorbox}

\section{Introduction}
\label{sec:intro}

Reinforcement learning (RL) has become a central post-training tool for improving
the reasoning and alignment behavior of large language models
(LLMs)~\citep{ouyang2022training,guo2025deepseekr1,grpo}. 
In this setting, an LLM is optimized as an autoregressive token-level policy
against scalar sequence-level rewards.
In practice, however, the data used for optimization is not generated by exactly
the same policy being updated.
Rollouts are often produced by inference engines whose numerical behavior differs
from the training stack, creating a training-inference mismatch~\citep{qi2025defeating,yao2025offpolicy}.
Policy staleness adds another mismatch, since the policy is typically updated multiple
times on minibatches drawn from a fixed batch of rollout data~\citep{zheng2025stabilizing}.
Together, these effects make practical LLM RL inherently \emph{off-policy}.

Widely used methods~\citep{ppo,grpo,yu2025dapo,cispo,drgrpo} therefore follow a
common trust-region recipe of optimizing a clipped or masked surrogate objective
on rollouts collected by the behavior policy.
Such masks are \emph{asymmetric}.
Being outside the trust region is not sufficient for masking.
A token is masked only when its advantage-weighted update would further increase
the relevant deviation from the behavior policy; updates that reduce this
deviation remain active.
Thus a mask has two components:
the \emph{proximity} criterion asks whether the trust region is exceeded, and
the \emph{direction} criterion asks whether the update points further outward.
In Proximal Policy Optimization (PPO)~\citep{ppo}, both decisions are made from
the sampled token's importance ratio.
Once the ratio leaves a fixed window around one and the advantage pushes it
further away, the token is masked and its gradient vanishes.

The sampled importance ratio, however, is a poor proxy for distributional shift
in LLMs~\citep{dppo}.
It observes only the sampled token and ignores how probability mass moves over
the rest of the vocabulary.
Recent divergence-mask methods such as DPPO therefore replace the ratio-based
proximity criterion with a divergence between the behavior and training
policies~\citep{dppo}.
This gives a more faithful trust region, typically estimated from the top-$K$
log-probabilities exposed by rollout engines.
This upgrade, however, changes only the proximity criterion.
The direction criterion remains inherited from PPO.
For PPO, this direction criterion is consistent with the proximity criterion,
because both are defined by the sampled importance ratio.
For a divergence mask, however, the proximity criterion is defined by a
distributional divergence.
Whether the next update increases that divergence depends not only on the
sampled token but also on how the whole softmax distribution changes.
Thus DPPO makes the proximity criterion distributional, but leaves the direction
criterion single-sample.

To address the aforementioned inconsistency, we define the direction criterion by asking whether the next gradient
step will \emph{increase the divergence}, and mask a token according to the sign
of this predicted change.
We call the resulting rule the \emph{predictive divergence mask}.
For the softmax policies that every autoregressive LLM uses, this first-order
change has a clean closed form whose coefficient is a simple function of the
current token probabilities.
The coefficient splits into two terms.
A \emph{local} term at the sampled token is exactly the quantity the
ratio-based direction criterion already reads.
A \emph{global} term captures the softmax normalization coupling across the
vocabulary, which a single sampled ratio cannot observe.
Keeping only the local term recovers the ratio-based direction criterion, while
keeping both lets the sign track the true change of the divergence.
In practice, rollout engines usually expose only top-$K$ token probabilities
rather than the full vocabulary distribution.
We therefore adapt the predictive criterion to this truncated view with two
lightweight estimators for the unseen tail.
They correspond to an aggregated-tail approximation and a uniform-tail
approximation, respectively.
Both estimators use only simple reductions over the retained probabilities and
add negligible computational overhead.
Empirically, detailed analysis shows that the direction
criterion used by the predictive divergence mask is better aligned with the
realized change of the divergence than the sampled-ratio direction criterion.
Across model scales and precision settings, the predictive divergence mask
improves training stability and effectiveness over the baselines.

\section{Background}
\label{sec:background}

\noindent\textbf{LLM RL as token-level policy optimization.}
We follow the standard setup for RL fine-tuning of autoregressive language
models~\citep{ouyang2022training,grpo,dppo}. Given a prompt $x$, the policy $\tpi$
generates a response $y=(y_1,\dots,y_T)$ token by token.
We write $s_t=(x,y_1,\dots,y_{t-1})$ for the state at step $t$, so that
$\tpi(y\mid x)=\prod_{t=1}^{T}\tpi(y_t\mid s_t)$. A scalar reward $R(x,y)$ is
returned after the full response, and the objective is
$\mathcal{J}(\tpi)=\E_{y\sim\tpi}[R(x,y)]$. Responses are sampled from a \emph{behavior} (or rollout) policy $\bmu$,
typically a slightly stale copy of the current policy served by an optimized
inference engine.
The per-token importance ratio is
\begin{equation}
\label{eq:ratio}
r_t \;=\; \frac{\tpi(y_t\mid s_t)}{\bmu(y_t\mid s_t)}.
\end{equation}
Critic-free estimators such as GRPO~\citep{grpo} and its variants
\citep{ahmadian2024back,drgrpo} form the advantage $\Adv_t$ from group-relative
returns.
We treat $\Adv_t$ as given and inherit it unchanged from the underlying
algorithm.

\noindent\textbf{PPO and the ratio-based trust region.}
PPO~\citep{ppo} maximizes a clipped surrogate,
\begin{equation}
\label{eq:ppo}
L^{\mathrm{PPO}}(\tpi)=\E_{y\sim\bmu}\!\left[\sum_{t=1}^{|y|}
\min\!\Big(r_t\,\Adv_t,\;\operatorname{clip}(r_t,1-\eps,1+\eps)\,\Adv_t\Big)\right],
\end{equation}
which can equivalently be written as a masked policy-gradient objective
\begin{equation}
\label{eq:ppo-mask}
\begin{aligned}
&L^{\mathrm{PPO}}(\tpi)=\E_{y\sim\bmu}\!\left[\sum_{t=1}^{|y|} M_t^{\mathrm{PPO}}\,r_t\,\Adv_t\right],
\\[4pt]
&\text{where}\quad M_t^{\mathrm{PPO}}=
\begin{cases}
0, & \sign\big(\Adv_t\,(r_t-1)\big)>0 \;\text{ and }\; |r_t-1|>\eps,\\[2pt]
1, & \text{otherwise.}
\end{cases}
\end{aligned}
\end{equation}
The two factors of the mask $M_t^{\mathrm{PPO}}$ are the
\emph{proximity} criterion $|r_t-1|>\eps$ and the \emph{direction} criterion
$\sign\big(\Adv_t\,(r_t-1)\big)>0$, where $\sign(x)\in\{-1,0,+1\}$ returns the sign of $x$.
The proximity criterion measures how far the
training policy is from the behavior policy, but only through a single sampled token. Notably,
$D_{\mathrm{TV}}\big(\bmu(\cdot\mid s_t)\,\|\,\tpi(\cdot\mid s_t)\big)=
\tfrac12\,\E_{y_t\sim\bmu}\big[\,|r_t-1|\,\big]$,
which indicates that $|r_t-1|/2$ is a one-sample Monte Carlo estimate of the total variation distance.
Thus, PPO in effect gates on a noisy single-sample estimate of a
divergence~\citep{dppo}. 
The direction criterion encodes the asymmetric nature of trust-region control.
The mask acts only on outward-pointing updates.
Only updates that drive the policy \emph{further} from the behavior policy
threaten the trust region and need to be restrained.
Updates that pull the policy back toward it relax the trust-region constraint
and should be kept by the mask.
The sign of $\Adv_t(r_t-1)$ identifies whether the update is outward-pointing at
the sampled token.
It is positive exactly when $\Adv_t>0$ raises a probability already above the
behavior policy ($r_t>1$), or when $\Adv_t<0$ lowers a probability already below
it ($r_t<1$).
The PPO mask therefore removes only token updates that are both outside the trust
region and outward-pointing, while keeping corrective updates.

\noindent\textbf{DPPO and the divergence-based trust region.}
DPPO~\citep{dppo} keeps this masked structure and its asymmetric direction
criterion, but rebuilds the proximity criterion around the full divergence between
the two policies. Instead of gating on the single-sample estimate $|r_t-1|>\eps$, it
measures the per-token divergence directly,
$\Dtok_t=\Dtok\big(\bmu(\cdot\mid s_t)\,\|\,\tpi(\cdot\mid s_t)\big)$,
and uses $\Dtok_t>\delta$ as the proximity criterion.
This gives
\begin{equation}
\label{eq:dppo}
\begin{aligned}
&L^{\mathrm{DPPO}}(\tpi)=\E_{y\sim\bmu}\!\left[\sum_{t=1}^{|y|} M_t\,r_t\,\Adv_t\right],
\\[4pt]
&\text{where}\quad M_t=
\begin{cases}
0, & \sign\big(\Adv_t\,(r_t-1)\big)>0 \;\text{ and }\; \Dtok_t>\delta,\\[2pt]
1, & \text{otherwise.}
\end{cases}
\end{aligned}
\end{equation}
A common choice, and the one we use throughout, sets
$\Dtok=\KL(\bmu\,\|\,\tpi)$, the forward KL, which controls total variation
through Pinsker's inequality.

The full-vocabulary KL is not computable from a production rollout, which exposes
only the $K$ largest log-probabilities per step. The top-$K$ divergence addresses
this by restricting attention to a reduced support
$\Keep=\operatorname{TopK}(\bmu(\cdot\mid s_t),K)\cup\{y_t\}$, the $K$ highest-probability
tokens under the behavior policy augmented with the sampled token $y_t$ when it falls
outside the top $K$, and collapsing everything outside
$\Keep$ into a single aggregated tail bucket carrying the residual mass
$\bmu_{\mathrm{tail}}=1-\sum_{i\in\Keep}\bmu_i$ (and likewise $\tpi_{\mathrm{tail}}$). With this reduction,
\begin{equation}
\label{eq:topk-kl}
\KLtopk(\bmu\,\|\,\tpi)
\;=\; \sum_{i\in\Keep}\bmu_i\log\frac{\bmu_i}{\tpi_i}
\;+\; \bmu_{\mathrm{tail}}\log\frac{\bmu_{\mathrm{tail}}}{\tpi_{\mathrm{tail}}},
\end{equation}
which is a lower bound on the true KL (by the log-sum inequality) and, since
top-$K$ typically captures more than $99\%$ of the probability mass, an extremely
tight one~\citep{dppo}. Instantiating the mask of Eq.~(\ref{eq:dppo}) with this top-$K$ divergence, $\Dtok=\KLtopk$, gives the DPPO-TopK-KL method that we take as the starting point and primary baseline for this work.

DPPO upgrades only the proximity criterion.
The direction criterion is inherited from PPO unchanged, still flagging a token as
eligible for masking exactly when
$\sign\big(\Adv_t\,(r_t-1)\big)>0$.
The next section examines whether this sampled-ratio direction criterion tracks
the sign of the divergence change, and derives a replacement from the divergence
itself.

\section{Methodology}
\label{sec:method}

The divergence mask of Eq.~(\ref{eq:dppo}) upgrades the trust-region
\emph{proximity} criterion from a sampled-ratio proxy to a distributional
divergence, but leaves the
\emph{direction} criterion $\sign(\Adv_t(r_t-1))$ untouched. This section makes the
direction criterion consistent with the divergence-based proximity criterion. The
principle is to ask whether the next gradient step would actually increase the
divergence (\S\ref{sec:method-principle}).
We answer it with the divergence's own directional derivative, which we derive in
closed form (\S\ref{sec:method-derivation})
and then recover from the truncated top-$K$ statistics a rollout reports
(\S\ref{sec:method-estimators}). The result is a drop-in replacement for the direction
criterion, together with an analysis of when it agrees with the ratio-based one and when
it does not (\S\ref{sec:method-mask}).

\subsection{From ``will the ratio move away from one?'' to ``will the divergence grow?''}
\label{sec:method-principle}

The direction criterion determines whether keeping a token update would drive
the policy \emph{further} outside the trust region.
Under a ratio-based trust region, this question is about $|r_t-1|$.
The criterion $\sign(\Adv_t(r_t-1))$ answers it exactly because it is computed from the same ratio that defines the trust region.
Its sign therefore reports whether the update increases $|r_t-1|$.
Under a divergence-based trust region, the proximity criterion is instead defined by $\Dtok_t$, which aggregates the gap between the two policies over the entire vocabulary, while $\sign(\Adv_t(r_t-1))$ still looks only at the single sampled token. 
Its sign can therefore disagree with the sign of the actual change in $\Dtok_t$, 
so the ratio-based direction criterion no longer reliably indicates whether the update drives the policy further outside the trust region.

We therefore define the direction criterion directly in terms of $\Dtok_t$.
It asks whether the next gradient step increases $\Dtok_t$.
Let $\tpi_\eta$ denote the training policy after moving its logits along the
surrogate gradient by an effective step $\eta\ge0$.
Expanding the divergence along this step,
\begin{equation}
\label{eq:taylor-def}
\Dtok_t(\tpi_\eta) \;=\; \Dtok_t \;+\; \eta\cdot \frac{d}{d\eta}\Dtok_t(\tpi_\eta)\Big|_{\eta=0} \;+\; O(\eta^2),
\end{equation}
the sign of the first-order term determines the local trend of $\Dtok_t$.
A positive sign means the update increases $\Dtok_t$, and a negative sign means
it decreases $\Dtok_t$.
A token already outside the trust region should therefore be blocked when this
term is positive, since the update would push $\Dtok_t$ still higher, and kept
when it is negative, since the update would pull $\Dtok_t$ back down.
We replace the ratio-based direction criterion with the sign of this first-order
change, the divergence's own response to the update rather than a single-sample
proxy for it.

\subsection{The directional derivative of the divergence}
\label{sec:method-derivation}

We work at a single token and drop the index $t$ for brevity. Let
$\tpi_i = e^{z_i}/\sum_j e^{z_j}$ be the training policy over the vocabulary,
with logits $z_i$, and let $k$ be the sampled action (the token $y_t$ of the
underlying step). The per-token surrogate
objective $r\,\Adv$ has gradient $\Adv\,r\,\nabla_z\log\tpi_k$ with respect to the
logits, where $\nabla_{z_i}\log\tpi_k=\1[i=k]-\tpi_i$. A single gradient step therefore moves the
logits along the direction $\1[i=k]-\tpi_i$, scaled by $\Adv\,r$.
Since $r>0$, the \emph{sign} of the step is $\sign(\Adv)$.
We compute the directional derivative for a
unit step in $\bm{v}_i=\1[i=k]-\tpi_i$ and recombine the advantage sign afterwards.

\begin{proposition}[Directional derivative of $\KL(\bmu\,\|\,\tpi)$ along a policy-gradient step]
\label{prop:dir-deriv}
Let $\tpi_\eta$ be the softmax policy whose logits are $z_i+\eta\,(\1[i=k]-\tpi_i)$.
Then
\begin{equation}
\label{eq:taylor-coeff}
\Taylor \;\equiv\; \frac{d}{d\eta}\,\KL(\bmu\,\|\,\tpi_\eta)\Big|_{\eta=0}
\;=\; (\tpi_k-\bmu_k) \;+\; \sum_i \tpi_i\,(\bmu_i-\tpi_i).
\end{equation}
\end{proposition}

We denote this directional derivative $\Taylor$ throughout and defer its derivation to
Appendix~\ref{app:derivations}.

\noindent\textbf{What the ratio-based direction criterion misses.}
The coefficient $\Taylor$ decomposes into two terms with distinct scope,
\begin{equation}
\label{eq:decomp}
\Taylor \;=\; \underbrace{(\tpi_k-\bmu_k)}_{\text{local: sampled token } k}
\;+\; \underbrace{\sum_i \tpi_i(\bmu_i-\tpi_i)}_{\text{global: normalization coupling}},
\end{equation}
and the split is exactly the boundary between what a single sample observes and what
it cannot. The local term is the probability gap at the sampled token. Since
$r-1=(\tpi_k-\bmu_k)/\bmu_k$ with $\bmu_k>0$,
\begin{equation}
\label{eq:local-is-ratio}
\sign(\tpi_k-\bmu_k)=\sign(r-1),
\end{equation}
so the ratio-based direction criterion is precisely the sign of the local term.
The decomposition shows exactly what this criterion misses.
Changing the sampled probability necessarily changes the rest of the distribution
through the softmax normalizer, creating the global term $\E_{i\sim\tpi}[\bmu_i-\tpi_i]$ that no
single sampled ratio can reveal.
The ratio-based direction criterion is therefore equivalent to using only the
local term and ignoring this coupling.
The divergence-based criterion keeps both terms, which is why its sign can track
the true change in $\Dtok_t$ where the ratio sign cannot.

\begin{remark}[Step-size independence]
\label{rem:step-size}
The actual logit step scales the unit direction $\bm v$ by $\eta\,\Adv\,r$,
where $\eta>0$ is the effective step size and $r>0$ is the importance ratio.
Thus the first-order change of the divergence has sign
$\sign(\Adv\cdot\Taylor)$.
The mask only needs this sign.
It does not need to estimate $\eta$ or the magnitude of $r$.
By contrast, predicting the post-update value $\Dtok_t(\tpi_\eta)$ would require
those quantities and the higher-order $O(\eta^2)$ terms.
The sign prediction is reliable when the first-order term dominates the
higher-order remainder.
\end{remark}

\subsection{Estimating the directional derivative under a truncated vocabulary}
\label{sec:method-estimators}

Computing $\Taylor$ exactly would require the full vocabulary distribution, whereas a
production rollout stores only the top-$K$ probabilities and the leftover tail mass.
The shortfall is minor, because the sampled token and the retained top-$K$ entries are
observed directly and the only missing ingredient is the contribution of the unseen
tail. Estimating $\Taylor$ therefore reduces to modeling that tail, and we
consider two estimators that differ only in how they do so.

\noindent\textbf{Aggregated-tail estimator.}
The simplest choice mirrors the bucketing already used by the top-$K$ divergence in
Eq.~(\ref{eq:topk-kl}) and collapses the entire tail into one aggregated token of mass
$\tpi_{\mathrm{tail}}$ (resp.\ $\bmu_{\mathrm{tail}}$). The coefficient
becomes
\begin{equation}
\label{eq:agg-tail}
\Taylor^{\mathrm{agg}}=(\tpi_k-\bmu_k)
+\sum_{i\in\Keep}\tpi_i(\bmu_i-\tpi_i)
+\tpi_{\mathrm{tail}}\big(\bmu_{\mathrm{tail}}-\tpi_{\mathrm{tail}}\big).
\end{equation}
This single-bucket approximation can overestimate the tail contribution by
concentrating all residual mass on one aggregate token.

\noindent\textbf{Uniform-tail estimator.}
A more faithful model spreads the residual mass evenly over the $n-m$ unseen
tokens, taking $\tpi_j\approx\tpi_{\mathrm{tail}}/(n-m)$ and
$\bmu_j\approx\bmu_{\mathrm{tail}}/(n-m)$
for $j\notin\Keep$ (with $m=|\Keep|$ and $n$ the vocabulary size). Substituting into
the tail sums gives
\begin{equation}
\label{eq:uni-tail}
\Taylor^{\mathrm{uni}}=(\tpi_k-\bmu_k)
+\sum_{i\in\Keep}\tpi_i(\bmu_i-\tpi_i)
+\frac{\tpi_{\mathrm{tail}}\big(\bmu_{\mathrm{tail}}-\tpi_{\mathrm{tail}}\big)}{n-m},
\end{equation}
the same tail term as Eq.~(\ref{eq:agg-tail}) but suppressed by $1/(n-m)$. Since the
vocabulary is large this term is negligible, in contrast to the exaggerated estimate of
the single-bucket model.

\noindent\textbf{The gap between the two estimators.}
Since the two estimators share all retained-token terms and differ only in their
tail contribution, their difference is
\begin{equation}
\label{eq:tail-gap}
\Taylor^{\mathrm{agg}}-\Taylor^{\mathrm{uni}}
=\Big(1-\tfrac{1}{n-m}\Big)\,\tpi_{\mathrm{tail}}\,(\bmu_{\mathrm{tail}}-\tpi_{\mathrm{tail}})
\;\approx\; \tpi_{\mathrm{tail}}\,(\bmu_{\mathrm{tail}}-\tpi_{\mathrm{tail}}),
\end{equation}
where the approximation uses $n-m\gg1$.
In LLMs, the vocabulary size $n$ is typically on the order of $10^5$, while
$m=|\Keep|$ is typically at most $20$.
Thus the gap is small whenever the retained top-$K$ tokens capture most of the
probability mass.
In this regime, the aggregated-tail and uniform-tail masks should behave nearly
identically.

Both estimators are computationally lightweight. Each reads only the top-$K$
probabilities and tail masses already returned by the rollout, requiring no
additional forward or backward pass and no finite-difference evaluation.
Appendix~\ref{app:binary-kl-equivalence} shows that, under DPPO's binary KL
approximation, the aggregated-tail divergence-based direction criterion reduces exactly to the
ratio-based direction criterion.
Appendix~\ref{app:topk-tv-predictive} gives the analogous divergence-based direction criterion
when the proximity criterion is top-$K$ TV.

\subsection{Policy optimization with the predictive divergence mask}
\label{sec:method-mask}

Replacing the ratio-based direction criterion in Eq.~(\ref{eq:dppo}) with the
divergence-based direction criterion built from the directional derivative $\Taylor_t$ gives
the \emph{predictive divergence mask}:
\begin{equation}
\label{eq:predictive-mask}
M_t^{\mathrm{pred}} =
\begin{cases}
0, & \sign\big(\Adv_t\cdot\Taylor_t\big)>0 \;\text{ and }\; \Dtok_t>\delta,\\[2pt]
1, & \text{otherwise,}
\end{cases}
\end{equation}
where $\Dtok_t=\KLtopk(\bmu\,\|\,\tpi)$ and $\Taylor_t$ is the directional derivative of
Eq.~(\ref{eq:taylor-coeff}) (estimated by either the aggregated-tail or uniform-tail form). The mask reuses only
quantities already on hand during training, namely $\Adv_t$, the top-$K$ probabilities,
and the divergence, and introduces no new hyperparameter beyond the existing
threshold $\delta$. The resulting objective is the masked policy gradient
\begin{equation}
\label{eq:predictive-obj}
\nabla_\theta L^{\mathrm{pred}}(\tpi)
= \E_{y\sim\bmu}\!\left[\sum_{t=1}^{|y|}
M_t^{\mathrm{pred}}\; r_t\,\Adv_t\;\nabla_\theta\log\tpi(y_t\mid s_t)\right],
\end{equation}
which is identical to the DPPO-TopK-KL update of Eq.~(\ref{eq:dppo}) except that its direction criterion follows the divergence's first-order change instead of the sampled ratio.

\begin{remark}[Local nature of the prediction]
The predictive mask is a token-level first-order approximation.
It estimates whether one token's gradient contribution would locally increase or
decrease that token's divergence.
The actual parameter update aggregates gradients from all tokens in the batch, so
the realized divergence change at one token can also be affected by other tokens'
updates.
Thus the predicted sign should be viewed as a local estimate of the divergence
trend, not an exact prediction of the post-update divergence.
\end{remark}

\noindent\textbf{Agreement and disagreement of the two criteria.}
The decomposition in Eq.~(\ref{eq:decomp}) can also be read as centering the
sampled-token gap by the policy-weighted average gap:
\[
\Taylor
=
(\tpi_k-\bmu_k)-\E_{i\sim\tpi}[\tpi_i-\bmu_i].
\]
Thus the global term changes the mask decision only if this centering crosses
zero.
Equivalently, disagreement occurs only when the current-policy average gap lies
farther from zero in the same direction as the sampled-token gap:
\begin{equation}
\label{eq:recovery}
0<\tpi_k-\bmu_k<\E_{i\sim\tpi}[\tpi_i-\bmu_i]
\quad\text{or}\quad
0>\tpi_k-\bmu_k>\E_{i\sim\tpi}[\tpi_i-\bmu_i].
\end{equation}
Together with the proximity condition $\Dtok_t>\delta$, these are precisely the
tokens on which the predictive mask differs from the ratio-based mask.
They are also the cases where a single-sample ratio is least informative: the
sampled-token gap points in one direction, but the distribution-wide correction
reverses the predicted change in divergence.
Otherwise, the centered gap keeps the local sign, and the two masks agree.

\section{Experiments and Results}
\label{sec:experiments}

\begin{figure}[t]
    \centering
    \includegraphics[width=\linewidth]{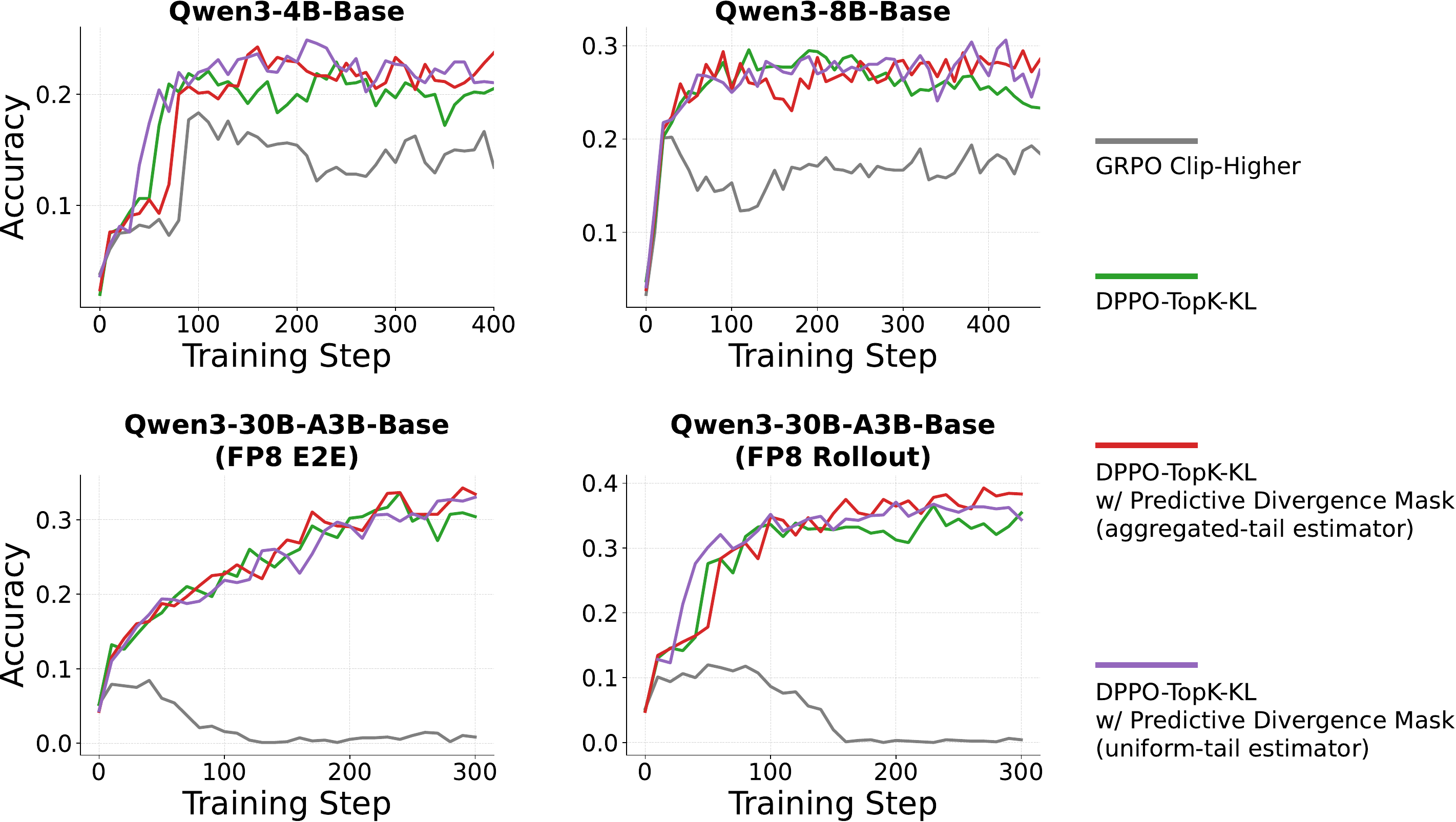}
    \caption{Evaluation accuracy (avg@16, averaged over AIME24 and AIME25) over
    training at $\delta=0.15$ on the four settings: Qwen3-4B-Base, Qwen3-8B-Base, Qwen3-30B-A3B-Base with
    FP8 E2E training, and Qwen3-30B-A3B-Base with FP8 Rollout. The two predictive divergence masks use the aggregated-tail and uniform-tail
estimators for the top-$K$ KL directional-derivative coefficient.}
    \label{fig:main_015}
\end{figure}

\begin{figure}[t]
    \centering
    \includegraphics[width=\linewidth]{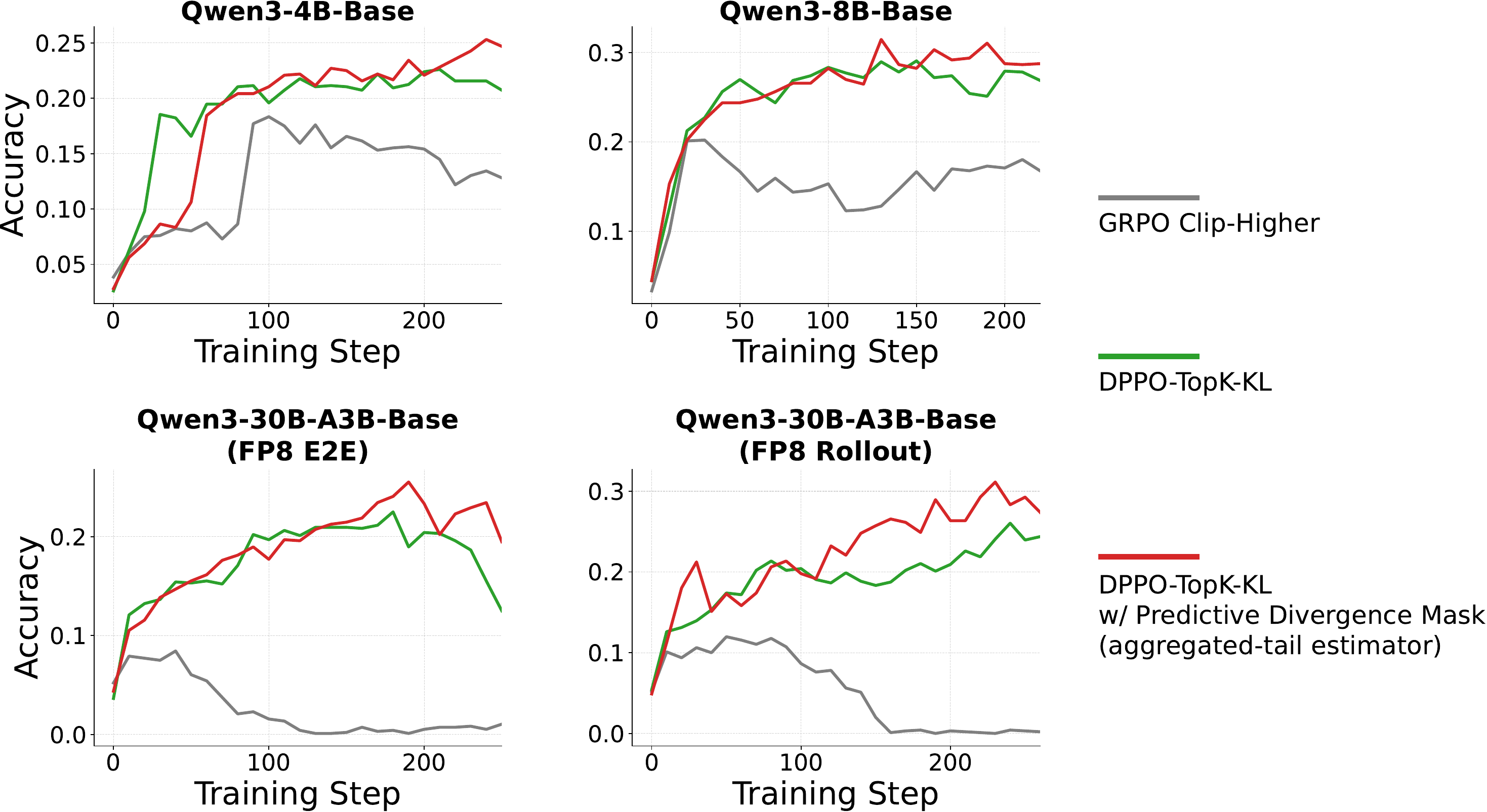}
    \caption{Evaluation accuracy (avg@16, averaged over AIME24 and AIME25) across the four settings with divergence threshold $\delta=0.05$.}
    \label{fig:main_005}
\end{figure}

\subsection{Main experiment}
\label{sec:exp-setup}

\noindent\textbf{Models, Data, and Benchmarks.} We use a filtered version of DAPO-Math-17k\footnote{\url{https://huggingface.co/datasets/Jiawei415/DPAO_filter}}~\citep{yu2025dapo} as the training dataset, which contains about 13k training samples. In evaluation, we use AIME24 and AIME25 as the benchmarks; for each problem we sample 16 responses and report the average accuracy (avg@16). For the base models, we choose Qwen3-4B-Base, Qwen3-8B-Base, and Qwen3-30B-A3B-Base~\citep{yang2025qwen3} for our experiments. For the Qwen3-30B-A3B-Base model, we use FP8 precision for rollout only (FP8 Rollout) or for both training and rollout (FP8 E2E) to accelerate the training process. FP8 also widens the numerical gap between the rollout and training engines, so these two configurations serve as a stress test. We build on the VeRL framework~\citep{sheng2024hybridflow} for RL training. Experiments are conducted on four nodes, each with 8 NVIDIA H20 GPUs (32 GPUs total).

\noindent\textbf{Methods and Baselines.}
Our primary baseline is DPPO-TopK-KL~\citep{dppo}, which pairs the top-$K$ KL proximity criterion with the original ratio-based direction criterion (Eq.~(\ref{eq:dppo})).
Against it we evaluate two predictive divergence masks, which set the direction from the aggregated-tail estimator (Eq.~(\ref{eq:agg-tail})) and the uniform-tail estimator (Eq.~(\ref{eq:uni-tail})).
We also include GRPO with the clip-higher trick ($\eps_{\mathrm{low}}=0.2$,
$\eps_{\mathrm{high}}=0.28$)~\citep{yu2025dapo} as a ratio-clipping reference. 
All divergence-based methods use $K=20$.
The main comparison uses the recommended $\delta=0.15$, and a threshold-sensitivity study uses $\delta=0.05$. 
Remaining hyperparameters and hardware details are given in Appendix~\ref{app:experiments} and Table~\ref{tab:hyperparameters}.

\noindent\textbf{Main results.}
Figure~\ref{fig:main_015} reports the main comparison at the recommended
threshold $\delta=0.15$.
Across all four model--precision settings, GRPO clip-higher is unstable and
collapses in the two Qwen3-30B-A3B-Base settings, while the divergence-based
methods remain stable.
Within the divergence-based family, the two predictive divergence masks improve over DPPO-TopK-KL.
This comparison isolates the effect of the direction criterion, since DPPO-TopK-KL uses the same top-$K$ KL proximity criterion and differs only in whether the direction is read from the sampled ratio or predicted from the divergence change.
The improvement is therefore consistent with the central hypothesis of the paper:
when the trust region is defined by a divergence, the direction criterion should track the divergence itself rather than the sampled-token ratio.

The aggregated-tail and uniform-tail masks behave very similarly during
training, as expected from the small tail correction of Eq.~(\ref{eq:tail-gap}).
This suggests that the predictive mask is not sensitive to the particular tail
model used to estimate the directional-derivative coefficient.
At the tighter threshold $\delta=0.05$, all methods perform worse, indicating
that this trust region is overly restrictive (Figure~\ref{fig:main_005}).
Even under this tight threshold, however, the predictive divergence masks still
improve over DPPO-TopK-KL, suggesting better robustness to the trust-region
hyperparameter.

\subsection{Analysis of direction criteria}
\label{sec:direction_analysis}

\noindent\textbf{Token-level setup.}
We next examine the direction criterion itself at the token level.
The predictive and ratio-based masks share the same top-$K$ KL proximity
criterion, so they can differ only after a token has already left the trust region.
We therefore restrict the analysis to tokens with
$D=\KLtopk(\bmu\,\|\,\tpi)>\delta$.
For such a token, the ratio-based direction criterion keeps the update when
$\Adv\,(r-1)\le 0$.
The divergence-based direction criterion keeps it when
$\Adv\cdot\Taylor\le 0$, where $\Taylor$ is the estimated first-order
change of the top-$K$ divergence.

To evaluate these decisions, we compare them against the realized divergence
change after one actual update on the same batch,
$\Delta D=D_{\text{post}}-D_{\text{pre}}$.
If $\Delta D>0$, the kept update pushed the token further outside the trust region and should have been masked.
If $\Delta D<0$, the divergence contracted and the kept update was corrective.
We run $61$ independent seeds.
Each seed uses a fresh rollout of $1024$ sequences and applies each mask to the same batch.
Unless otherwise stated, error rates are computed per seed and reported as mean $\pm$ standard error.

\begin{figure}[t]
    \centering
    \includegraphics[width=0.9\linewidth]{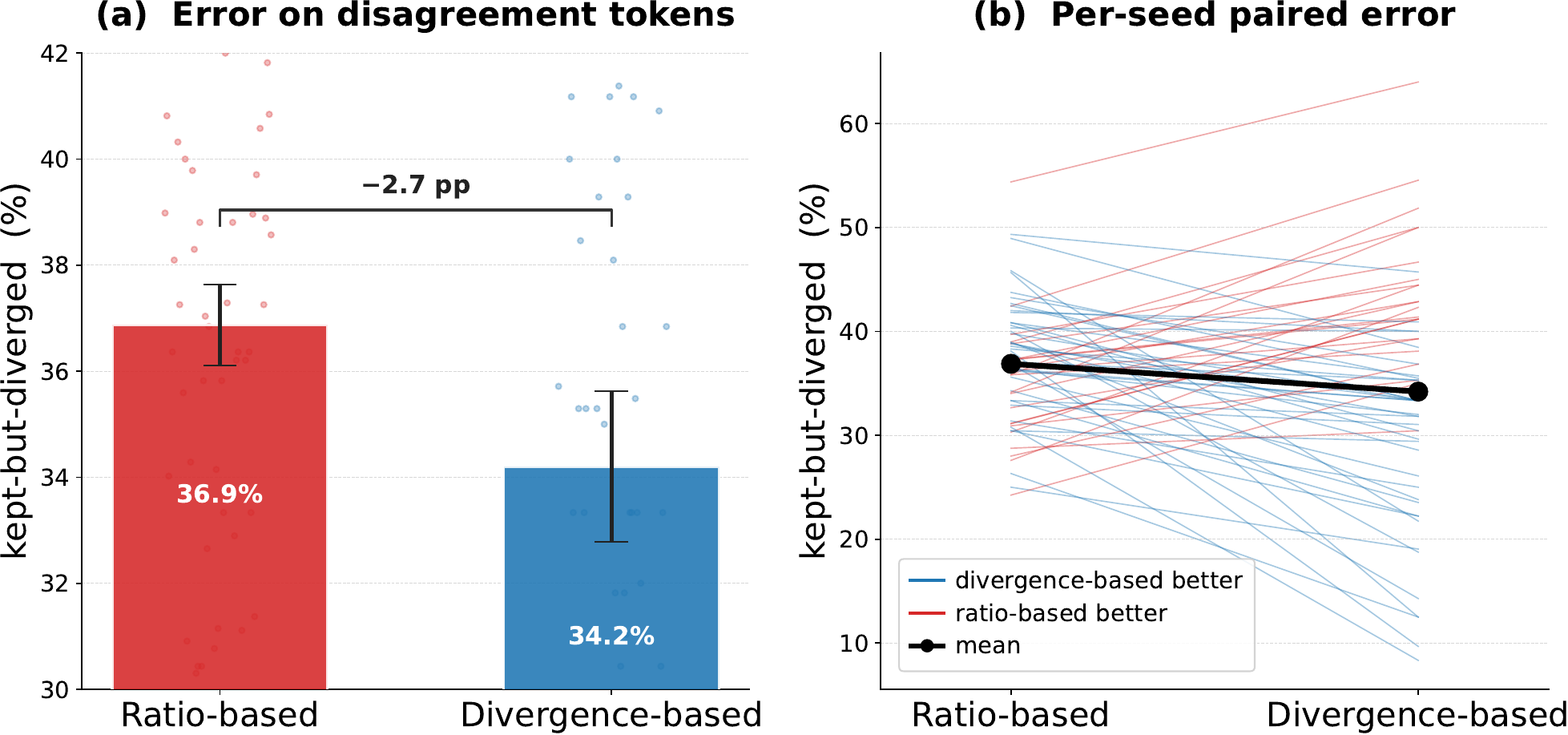}
    \caption{\textbf{Unsafe-keep rate on disagreement tokens}
    ($\delta=0.15$, $61$ seeds). We consider tokens outside the trust region
    ($D>\delta$) on which the ratio-based and divergence-based direction criteria
    make opposite keep/clip decisions. For each criterion, among the disagreement
    tokens it keeps, we report the fraction whose divergence increases after the
    update ($\Delta D>0$). Lower is better. \textbf{(a)} Mean unsafe-keep rate
    across seeds. The divergence-based direction criterion is lower than the
    ratio-based criterion ($34.2\%$ vs.\ $36.9\%$). Error bars show standard
    errors. \textbf{(b)} Per-seed paired comparison, with each line corresponding
    to one seed and colored by which criterion has the lower unsafe-keep rate.}
    \label{fig:disagreement_error}
\end{figure}

\begin{figure}[t]
    \centering
    \includegraphics[width=0.9\linewidth]{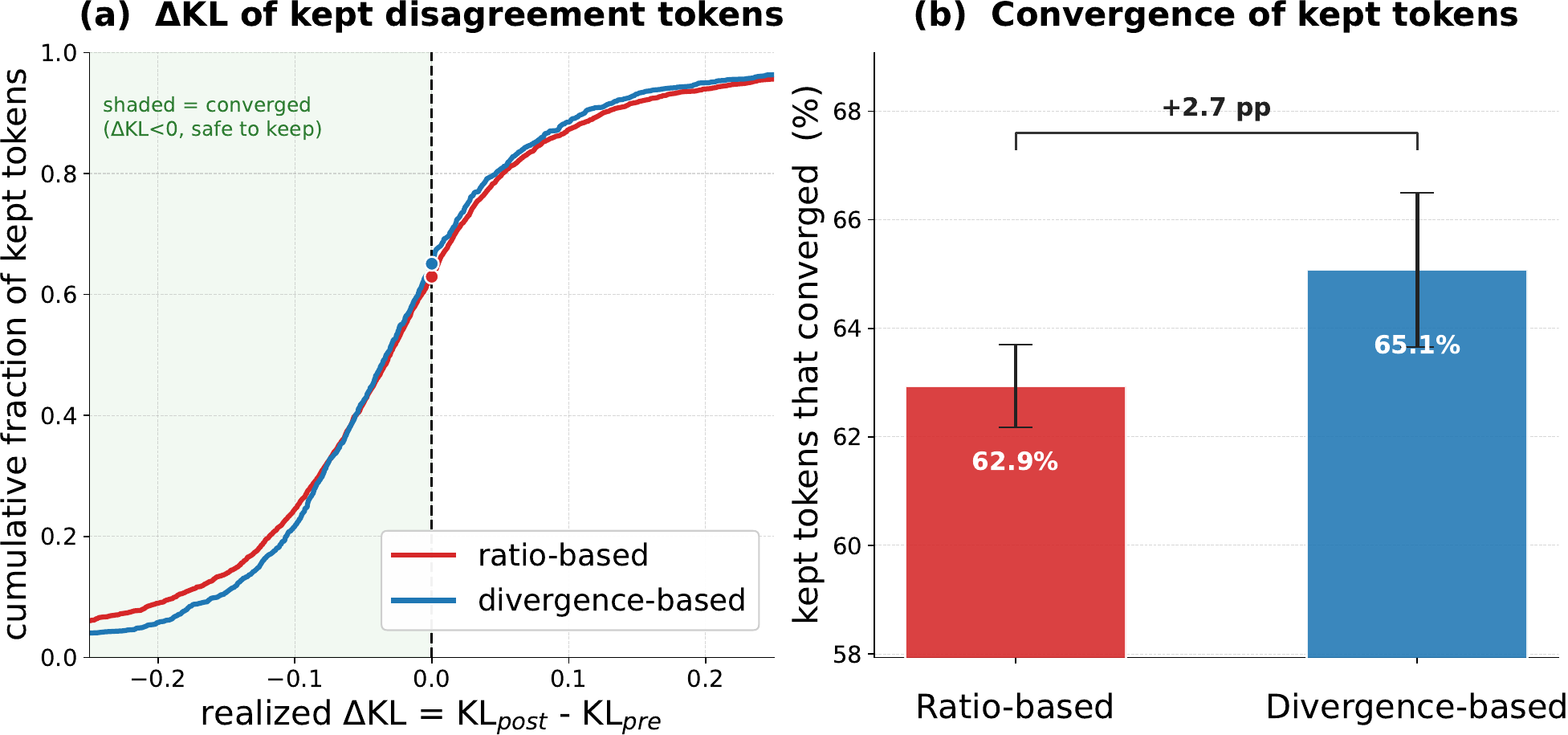}
    \caption{\textbf{Realized divergence change of kept disagreement tokens}
    ($\delta=0.15$, $61$ seeds). For each direction criterion, we collect the
    disagreement tokens it keeps and measure
    $\Delta D=D_{\text{post}}-D_{\text{pre}}$ after one actual update. Tokens from all
    seeds are combined in this figure. \textbf{(a)} Empirical CDF of $\Delta D$.
    The shaded region $\Delta D<0$ indicates kept updates that contracted the
    divergence. \textbf{(b)} Fraction of kept updates with $\Delta D<0$. The
    divergence-based direction criterion keeps a larger fraction of contracting
    updates than the ratio-based criterion ($65.1\%$ vs.\ $62.9\%$).}
    \label{fig:kept_dkl}
\end{figure}

\noindent\textbf{Results.}
On average, each seed contains $1435$ tokens outside the trust region.
The two direction criteria agree on the keep/clip decision for most of them and
disagree on only about $82$ tokens per seed.
Thus the two masks are distinguishable only on a small disagreement set, which is
where Figures~\ref{fig:disagreement_error}--\ref{fig:kept_dkl} focus.

Figure~\ref{fig:disagreement_error} measures unsafe keeps.
Among disagreement tokens kept by each criterion, the unsafe-keep rate is the
fraction whose divergence later increases.
Averaged over seeds, this rate is lower for the divergence-based direction
criterion than for the ratio-based direction criterion ($34.2\%$ vs.\ $36.9\%$),
a $2.7$ percentage-point reduction.
The divergence-based criterion has the lower unsafe-keep rate on $38$ of the
$61$ seeds.

Figure~\ref{fig:kept_dkl} shows the same comparison from the corrective side.
After combining tokens from all seeds, the fraction of kept updates that contract
the divergence is higher for the divergence-based direction criterion
($65.1\%$ vs.\ $62.9\%$).
This statistic is not the exact complement of the per-seed unsafe-keep rate in
Figure~\ref{fig:disagreement_error}, because it combines tokens across seeds before
computing the fraction.
The two figures therefore provide complementary views of the same directional
effect rather than two independent tests.

Overall, the token-level evidence supports the mechanism behind the predictive
mask.
The divergence-based direction criterion is better aligned with the realized
change of the divergence than the sampled-token ratio, although the effect is
modest because the disagreement set is small.

\section{Closing Remarks}

Divergence-based trust-region masks (e.g., DPPO) improve PPO's proximity criterion, but they
still inherit PPO's sampled-ratio direction criterion.
We address this mismatch with the predictive divergence mask, which decides
whether to mask an update from the predicted first-order change of the divergence
itself.
For softmax policies, this change decomposes into a local term that recovers the
ratio-based criterion and a global term that captures distribution-wide softmax
coupling.

Empirically, the divergence-based direction criterion is better aligned with the
realized divergence change, and the resulting mask improves
training stability and effectiveness across model scales and precision settings.
The method is still a local first-order approximation, since real parameter
updates aggregate gradients from all tokens and can introduce token interactions.
Nevertheless, the results suggest that the direction criterion should receive the
same distributional treatment as the proximity criterion.

\bibliography{ref}

\clearpage
\appendix
\section{Related Work}
\label{sec:related}

\noindent\textbf{Trust-region policy optimization.}
Constraining each update to a neighborhood of the data-generating policy is the
central idea of TRPO~\citep{trpo} and CPO~\citep{cpo}, which enforce an explicit
KL constraint, and of PPO~\citep{ppo}, which approximates it with the cheap
clipped surrogate of Eq.~(\ref{eq:ppo}). Mirror-descent and proximal
formulations~\citep{tomar2022mirror,wang2020truly,wang2019trust} make the
constraint explicit again, and SPO~\citep{spo} replaces the hard clip with a
smooth quadratic penalty whose stationary point lies on the trust-region
boundary. Our work stays within this family but separates a clipping rule into a
proximity criterion and a direction criterion. This separation shows that
existing trust-region masks can improve the proximity criterion while still using
PPO's sampled-ratio direction criterion.

\noindent\textbf{RL for LLM reasoning and off-policy correction.}
PPO-style optimization drives RLHF~\citep{ouyang2022training} and modern reasoning
training~\citep{guo2025deepseekr1,team2025kimi,zhou2025reinforcing,zhou2025variational}. GRPO~\citep{grpo} removes the
critic with group-relative advantages, and recent
follow-ups~\citep{ahmadian2024back,drgrpo,yu2025dapo,cispo} refine the advantage
estimator, the clipping range, or the importance-sampling correction.
DAPO~\citep{yu2025dapo} decouples the upper and lower clip bounds, while
CISPO~\citep{cispo} truncates importance weights instead of masking. These
methods modify the objective or the clipping rule, but they still decide the
update direction from the sampled importance ratio.
In deployed LLM RL systems, an additional source of instability is the
training-inference mismatch. The rollout policy served by an optimized inference
engine can differ from the training policy, which induces an off-policy gap and destabilizes
training~\citep{qi2025defeating,yao2025offpolicy,zheng2025stabilizing}.
Prior work mitigates this gap with truncated or masked importance
sampling~\citep{yao2025offpolicy,zheng2025stabilizing,liu-li-2025,team2025every},
or with engineering-level alignment of numerical and implementation
details~\citep{Team2025EveryAM,he2025nondeterminism}.

\noindent\textbf{Divergence-based trust-region masks.}
Closest to us are methods that replace the sampled ratio deviation $|r-1|$ with a
distributional measure of policy drift. DPPO~\citep{dppo} reinterprets PPO's clip
as a one-sample estimate of total variation and replaces it with an explicit
divergence mask, including the top-$K$ KL realization of Eq.~(\ref{eq:topk-kl})
that we build on.
Complementary work studies non-uniform token-level trust regions that vary the
trust-region constraint across tokens~\citep{mao2026beyond}.
Other divergence-based methods~\citep{r2vpo,yao2026rethinking} use smooth
regularization and reweight updates rather than masking them. All of these
methods inherit PPO's ratio-based direction criterion, which decides whether an
update is outward-pointing from the sign of $\Adv(r-1)$. Our contribution is to
make the direction criterion consistent with the divergence it serves.
We derive the first-order change of the divergence in closed form for softmax
policies and provide truncation-aware estimators of the resulting coefficient.

\clearpage
\section{Derivation of the Directional Derivative}
\label{app:derivations}

This appendix proves Proposition~\ref{prop:dir-deriv}.
The derivation is written for a generic softmax support.
It applies to the full vocabulary, and also to the reduced support used by the
top-$K$ construction when the tail is represented as an aggregate bucket.

\subsection{Softmax sensitivity}

We first record the first-order sensitivity of a softmax policy to a logit
perturbation, which underlies the entire derivation.

\begin{lemma}[Softmax logit sensitivity]
\label{lem:softmax-sensitivity}
Let $\tpi_i=e^{z_i}/\sum_j e^{z_j}$ be a softmax distribution with logits $z_i$, and let
$\tpi_{\eta}$ denote the distribution obtained by perturbing the logits along a fixed
direction $\bm v=(v_i)$, that is, $z_i(\eta)=z_i+\eta\,v_i$. Then
\begin{equation}
\label{eq:app-logderiv}
\frac{d}{d\eta}\log\tpi_{\eta,i}\Big|_{\eta=0}
= v_i - \sum_j \tpi_j v_j.
\end{equation}
\end{lemma}

\begin{proof}
By definition,
\begin{equation}
\log\tpi_{\eta,i}
= z_i+\eta v_i-\log\sum_j \exp(z_j+\eta v_j).
\end{equation}
Differentiating with respect to $\eta$ gives
\begin{equation}
\frac{d}{d\eta}\log\tpi_{\eta,i}
= v_i -
\frac{\sum_j v_j\exp(z_j+\eta v_j)}
{\sum_j \exp(z_j+\eta v_j)}.
\end{equation}
At $\eta=0$, the fraction is $\sum_j\tpi_j v_j$.
Thus,
\begin{equation}
\frac{d}{d\eta}\log\tpi_{\eta,i}\Big|_{\eta=0}
= v_i-\sum_j\tpi_j v_j .
\end{equation}
\end{proof}

\subsection{Directional derivative of the divergence}

\begin{proof}[Proof of Proposition~\ref{prop:dir-deriv}]
The proposition uses the unit logit direction induced by a policy-gradient step.
For the per-token surrogate $r\,\Adv$, the logit gradient is $\Adv\,r\,\nabla_z\log\tpi_k$, where
\begin{equation}
\nabla_{z_i}\log\tpi_k=\1[i=k]-\tpi_i .
\end{equation}
The scalar $r$ is positive, so the sign of the realized step is determined by
$\Adv$.
We therefore compute the derivative along the unit direction
\begin{equation}
\label{eq:app-v}
v_i = \1[i=k]-\tpi_i .
\end{equation}
Its average under the current policy is
\begin{equation}
\sum_j \tpi_j v_j
= \sum_j \tpi_j\big(\1[j=k]-\tpi_j\big)
= \tpi_k - \sum_j \tpi_j^2 .
\end{equation}

Now write the forward KL as
\begin{equation}
\KL(\bmu\|\tpi_\eta)
=
\sum_i \bmu_i\log\bmu_i
-
\sum_i \bmu_i\log\tpi_{\eta,i} .
\end{equation}
The first term is independent of $\eta$.
Applying Lemma~\ref{lem:softmax-sensitivity} to the second term gives
\begin{align}
\Taylor
&= -\sum_i \bmu_i\,
\frac{d}{d\eta}\log\tpi_{\eta,i}\Big|_{\eta=0}
\nonumber\\
&= -\sum_i \bmu_i\Big(v_i-\sum_j\tpi_j v_j\Big)
\nonumber\\
&= -\sum_i \bmu_i v_i
+ \Big(\sum_i\bmu_i\Big)\Big(\sum_j\tpi_j v_j\Big)
\nonumber\\
&= -\sum_i \bmu_i v_i + \sum_j\tpi_j v_j,
\label{eq:app-step}
\end{align}
where the last equality uses $\sum_i\bmu_i=1$.

The two terms in Eq.~(\ref{eq:app-step}) are
\begin{equation}
\sum_i \bmu_i v_i
= \sum_i \bmu_i\big(\1[i=k]-\tpi_i\big)
= \bmu_k-\sum_i\bmu_i\tpi_i,
\end{equation}
and
\begin{equation}
\sum_j\tpi_j v_j
= \tpi_k-\sum_j\tpi_j^2 .
\end{equation}
Substituting them into Eq.~(\ref{eq:app-step}) yields
\begin{align}
\Taylor
&= -\Big(\bmu_k-\sum_i\bmu_i\tpi_i\Big)
+ \Big(\tpi_k-\sum_j\tpi_j^2\Big)
\nonumber\\
&= (\tpi_k-\bmu_k)
+ \sum_i\tpi_i(\bmu_i-\tpi_i),
\label{eq:app-result}
\end{align}
which is Eq.~(\ref{eq:taylor-coeff}).
This completes the proof.
\end{proof}

\section{Equivalence under Binary-KL Approximation}
\label{app:binary-kl-equivalence}

DPPO also considers a binary KL approximation that collapses the vocabulary into
the sampled token and its complement. In this special case, the predictive
direction with the aggregated-tail estimator reduces exactly to the ratio-based
direction criterion.

Let the binary support contain the sampled token $k$ and a tail bucket. The two
policies are
\begin{equation}
\bmu^{\mathrm{bin}}=(\bmu_k, 1-\bmu_k),
\qquad
\tpi^{\mathrm{bin}}=(\tpi_k, 1-\tpi_k).
\end{equation}
The binary KL proximity criterion is
\begin{equation}
\KL\big(\bmu^{\mathrm{bin}}\|\tpi^{\mathrm{bin}}\big)
=
\bmu_k\log\frac{\bmu_k}{\tpi_k}
+
(1-\bmu_k)\log\frac{1-\bmu_k}{1-\tpi_k}.
\end{equation}
Applying Eq.~(\ref{eq:taylor-coeff}) to this two-atom support gives
\begin{align}
\Taylor^{\mathrm{bin}}
&= (\tpi_k-\bmu_k)
+ \tpi_k(\bmu_k-\tpi_k)
+ (1-\tpi_k)\big((1-\bmu_k)-(1-\tpi_k)\big)
\nonumber\\
&= (\tpi_k-\bmu_k)
- \tpi_k(\tpi_k-\bmu_k)
+ (1-\tpi_k)(\tpi_k-\bmu_k)
\nonumber\\
&= 2(1-\tpi_k)(\tpi_k-\bmu_k).
\label{eq:binary-kl-taylor}
\end{align}
For nondegenerate softmax probabilities, $0<\tpi_k<1$ and $\bmu_k>0$.
Therefore
\begin{equation}
\sign\big(\Taylor^{\mathrm{bin}}\big)
=
\sign(\tpi_k-\bmu_k)
=
\sign(r-1),
\end{equation}
where $r=\tpi_k/\bmu_k$.
It follows that
\begin{equation}
\sign\big(\Adv\,\Taylor^{\mathrm{bin}}\big)
=
\sign\big(\Adv\,(r-1)\big).
\end{equation}

Thus, if the proximity criterion is also the same binary KL, the predictive mask
with the aggregated-tail estimator is identical to DPPO-Binary-KL. This is a
degenerate case in which the binary approximation removes the distributional
degrees of freedom that the divergence-based direction criterion is designed to
use. The top-$K$ KL setting retains additional tokens, so the global correction
can differ from the sampled-token ratio direction.

\section{Predictive Divergence Masks under TV Divergence}
\label{app:topk-tv-predictive}

The paper focuses on top-$K$ forward KL because it matches the primary baseline
DPPO-TopK-KL.
DPPO can also use total variation as the proximity criterion.
This section derives the corresponding divergence-based direction criterion for TV using the same
first-order principle as in the KL case.
For a generic support, define
\begin{equation}
D_{\mathrm{TV}}(\bmu,\tpi)
=
\frac12\sum_i|\bmu_i-\tpi_i|.
\end{equation}
The divergence is not differentiable when a probability gap is exactly zero.
We use the usual subgradient convention $\sign(0)=0$.
On any fixed support, Lemma~\ref{lem:softmax-sensitivity} gives the
directional derivative of TV along the unit policy-gradient direction:
\begin{equation}
\label{eq:tv-generic-deriv}
\Taylor_{\mathrm{TV}}
=
-\frac12\sum_i s_i\,\tpi_i\Big(v_i-\sum_j\tpi_jv_j\Big),
\quad\text{where}\quad
s_i=\sign(\bmu_i-\tpi_i).
\end{equation}
The only difference from the KL case is the factor
$s_i=\sign(\bmu_i-\tpi_i)$, which comes from differentiating the absolute value
in TV.

\noindent\textbf{Aggregated-tail estimator.}
Under the top-$K$ aggregated-tail construction, the support contains the retained
tokens $\Keep$ and one tail bucket.
This treats the unseen tail as a single coordinate when computing both the TV
proximity and the softmax coupling term.
The top-$K$ TV proximity criterion is as follows:
\begin{equation}
D_{\mathrm{TV}}^{\mathrm{TopK}}(\bmu,\tpi)
=
\frac12\sum_{i\in\Keep}|\bmu_i-\tpi_i|
+
\frac12|\bmu_{\mathrm{tail}}-\tpi_{\mathrm{tail}}|.
\end{equation}
For the tail bucket, define
\begin{equation}
 s_{\mathrm{tail}}
 =\sign(\bmu_{\mathrm{tail}}-\tpi_{\mathrm{tail}}).
\end{equation}
On this reduced support,
\begin{equation}
 c_{\mathrm{agg}}
 =\sum_j\tpi_jv_j
 =\tpi_k-
 \sum_{i\in\Keep}\tpi_i^2
 -\tpi_{\mathrm{tail}}^2 .
\end{equation}
Substituting into Eq.~(\ref{eq:tv-generic-deriv}) gives
\begin{equation}
\label{eq:tv-agg-tail}
\Taylor_{\mathrm{TV}}^{\mathrm{agg}}
=
-\frac12
\left[
\sum_{i\in\Keep}
 s_i\tpi_i\big(\1[i=k]-\tpi_i-c_{\mathrm{agg}}\big)
+
 s_{\mathrm{tail}}\tpi_{\mathrm{tail}}
 \big(-\tpi_{\mathrm{tail}}-c_{\mathrm{agg}}\big)
\right].
\end{equation}

\noindent\textbf{Uniform-tail estimator.}
Under the top-$K$ uniform-tail construction, the residual mass is spread over the
$n-m$ unseen tokens, where $m=|\Keep|$ and $n$ is the vocabulary size.
This keeps the observed head terms unchanged but models the unseen tail as many
small coordinates rather than one aggregate coordinate:
\begin{equation}
\tpi_j\approx\frac{\tpi_{\mathrm{tail}}}{n-m},
\qquad
\bmu_j\approx\frac{\bmu_{\mathrm{tail}}}{n-m},
\qquad j\notin\Keep.
\end{equation}
The softmax coupling term becomes
\begin{equation}
 c_{\mathrm{uni}}
 =\tpi_k-
 \sum_{i\in\Keep}\tpi_i^2
 -\frac{\tpi_{\mathrm{tail}}^2}{n-m}.
\end{equation}
The tail contribution in Eq.~(\ref{eq:tv-generic-deriv}) then sums to
\begin{equation}
 s_{\mathrm{tail}}
 \left(
 -\frac{\tpi_{\mathrm{tail}}^2}{n-m}
 -c_{\mathrm{uni}}\tpi_{\mathrm{tail}}
 \right).
\end{equation}
Thus
\begin{equation}
\label{eq:tv-uni-tail}
\Taylor_{\mathrm{TV}}^{\mathrm{uni}}
=
-\frac12
\left[
\sum_{i\in\Keep}
 s_i\tpi_i\big(\1[i=k]-\tpi_i-c_{\mathrm{uni}}\big)
+
 s_{\mathrm{tail}}
 \left(
 -\frac{\tpi_{\mathrm{tail}}^2}{n-m}
 -c_{\mathrm{uni}}\tpi_{\mathrm{tail}}
 \right)
\right].
\end{equation}

The masking rule is unchanged.
Only the proximity measure and the divergence-based direction coefficient are replaced
by their TV counterparts.
With either estimator, the corresponding predictive TV mask is as follows:
\begin{equation}
M_t^{\mathrm{pred}} =
\begin{cases}
0, & \sign\big(\Adv_t\,\Taylor_{\mathrm{TV},t}\big)>0
\ \text{ and }\ D_{\mathrm{TV},t}^{\mathrm{TopK}}>\delta,\\[2pt]
1, & \text{otherwise.}
\end{cases}
\end{equation}
Since the paper focuses on DPPO-TopK-KL, this TV form is included as the
corresponding extension for a TV proximity criterion.

\section{Additional Experimental Details}
\label{app:experiments}

\subsection{Training Details}
\label{app:training-config}

\noindent\textbf{Implementation.}
We implement the training pipeline in VeRL~\citep{sheng2024hybridflow}, using
Megatron-LM~\citep{megatron-lm} for training and vLLM~\citep{kwon2023efficient}
for rollout generation. By default, both training and rollout use BF16.
For Qwen3-30B-A3B-Base, we evaluate two low-precision settings:
FP8 only in vLLM (FP8 Rollout) and FP8 in both Megatron-LM and vLLM (FP8 E2E).

\noindent\textbf{Hyperparameters.}
All methods use the same learning rate,
batching, rollout temperature, response budget, and group size. Each rollout
batch is used for one PPO epoch. The DPPO-TopK-KL and predictive runs use a
top-$K$ support of $K=20$; the main comparison uses $\delta=0.15$, and the
tight-threshold comparison uses $\delta=0.05$. Table~\ref{tab:hyperparameters}
summarizes the hyperparameters used in our experiments.

\begin{table}[t]
    \centering
    \small
    \caption{Training hyperparameters. Both Qwen3-30B-A3B-Base precision settings use
    the values in the final column.}
    \label{tab:hyperparameters}
    \begin{tabular}{@{}lccc@{}}
    \toprule
        \textbf{Hyperparameter} & \textbf{Qwen3-4B-Base} & \textbf{Qwen3-8B-Base} & \textbf{Qwen3-30B-A3B-Base} \\
    \midrule
        Learning rate & $1\!\times\!10^{-6}$ & $1\!\times\!10^{-6}$ & $1\!\times\!10^{-6}$ \\
        PPO epochs & 1 & 1 & 1 \\
        Maximum prompt length & 2,048 & 2,048 & 2,048 \\
        Maximum response length & 8,192 & 8,192 & 8,192 \\
        Training batch size & 64 & 128 & 128 \\
        PPO minibatch size & 16 & 16 & 16 \\
        Rollout temperature & 1.0 & 1.0 & 1.0 \\
        Group size & 8 & 16 & 16 \\
        Top-$K$ support size, $K$ & 20 & 20 & 20 \\
        DPPO threshold, $\delta$ & 0.15 / 0.05 & 0.15 / 0.05 & 0.15 / 0.05 \\
    \bottomrule
    \end{tabular}
\end{table}

\subsection{Detailed training dynamics}
\label{app:training-dynamics}

Figures~\ref{fig:4b_full}--\ref{fig:30ba3b_fp8rollout_full} show additional
training details, including training reward, separate AIME24 and AIME25
accuracy, response length, PPO-KL, and clip fraction.

DPPO-TopK-KL and the two predictive divergence masks occupy a similar PPO-KL range, while the predictive divergence masks generally exhibit lower clip fractions.

\begin{figure}[h]
    \centering
    \includegraphics[width=0.94\linewidth]{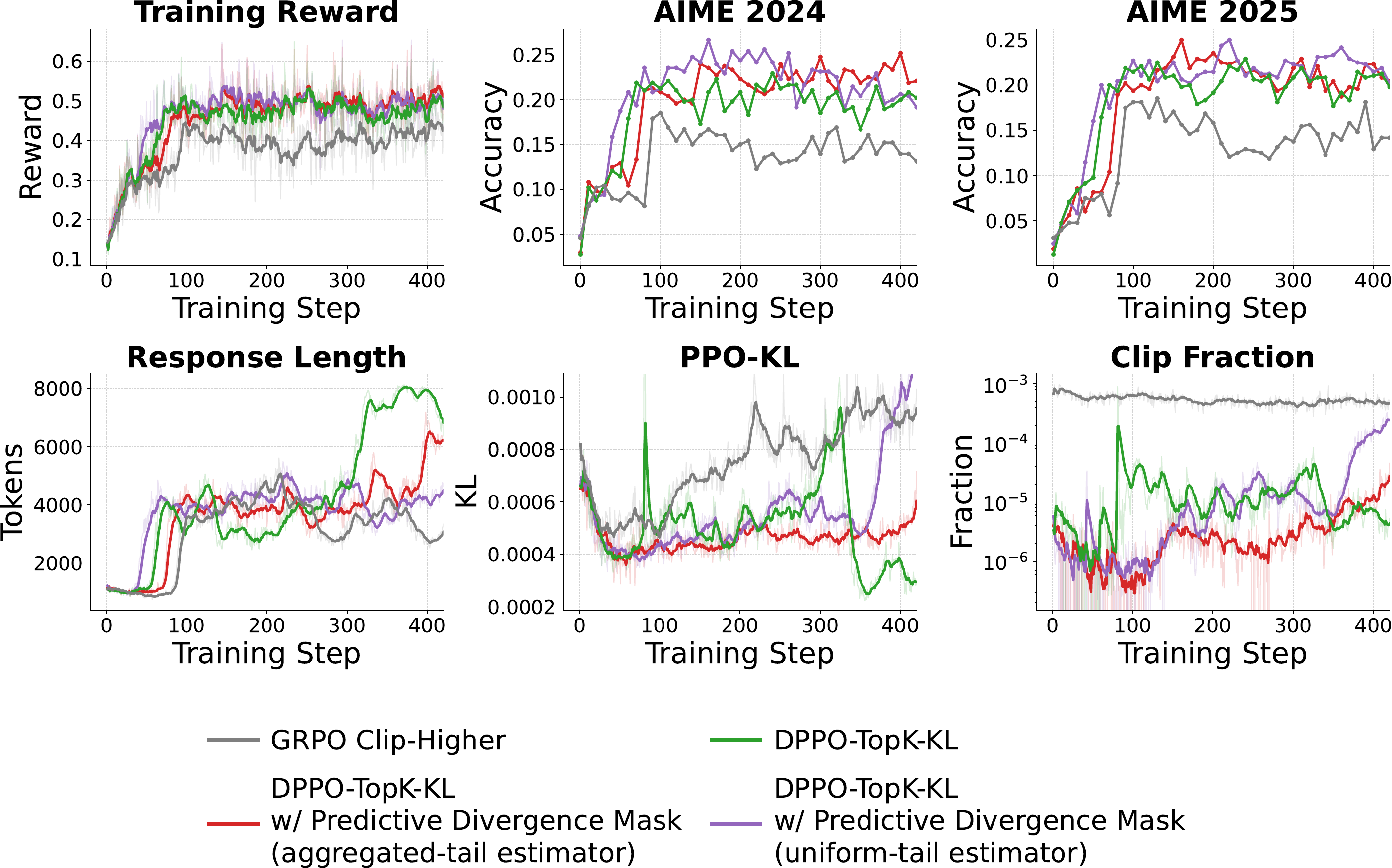}
    \caption{Training dynamics for Qwen3-4B-Base at $K=20$ and
    $\delta=0.15$. AIME24 and AIME25 are shown separately. The remaining
    panels report reward, response length, PPO-KL, and clip fraction.}
    \label{fig:4b_full}
\end{figure}

\begin{figure}[h]
    \centering
    \includegraphics[width=0.94\linewidth]{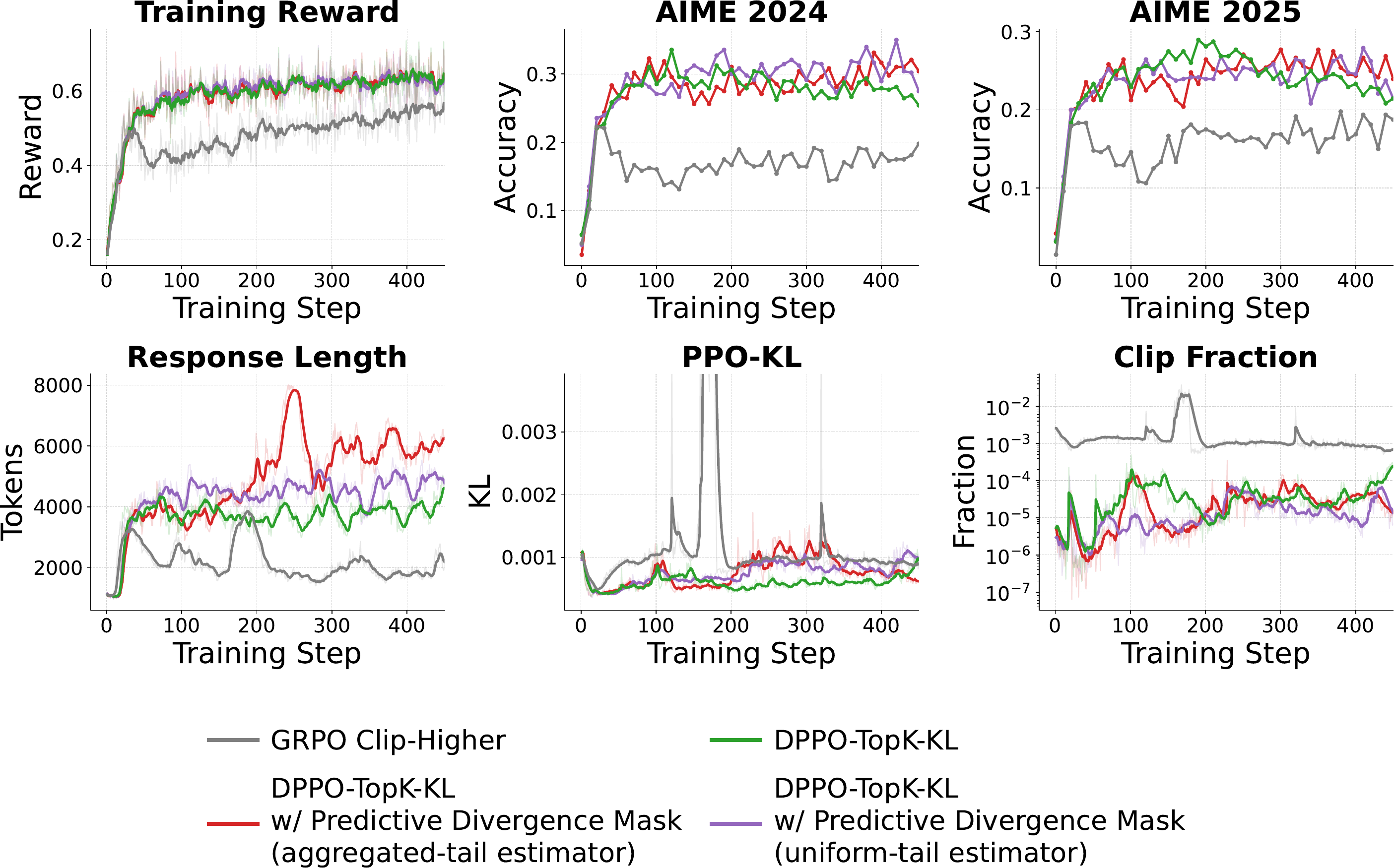}
    \caption{Training dynamics for Qwen3-8B-Base at $K=20$ and
    $\delta=0.15$, with the same metrics and methods as
    Figure~\ref{fig:4b_full}.}
    \label{fig:8b_full}
\end{figure}

\begin{figure}[h]
    \centering
    \includegraphics[width=0.94\linewidth]{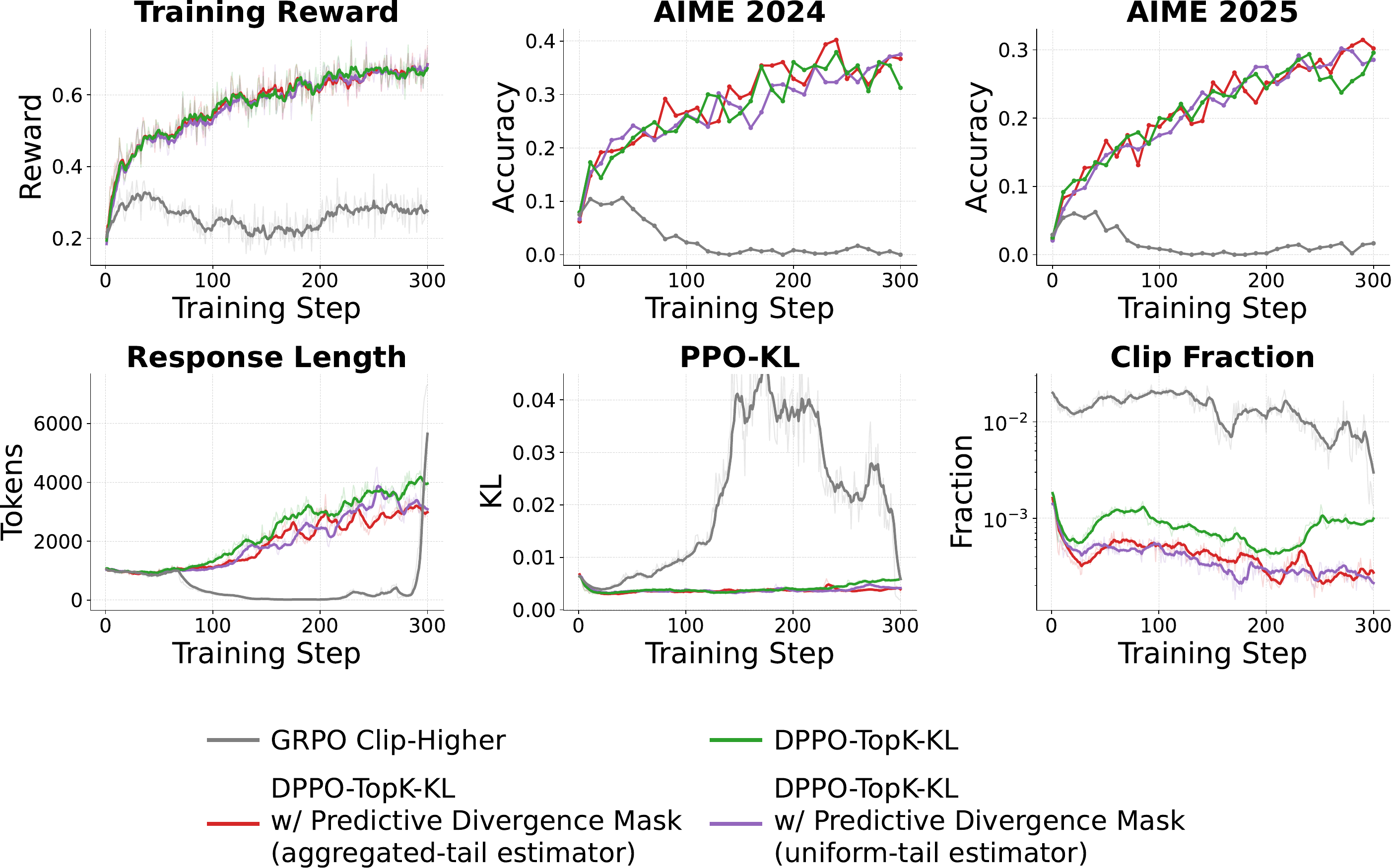}
    \caption{Training dynamics for Qwen3-30B-A3B-Base under FP8 E2E training at
    $K=20$ and $\delta=0.15$.}
    \label{fig:30ba3b_fp8e2e_full}
\end{figure}

\begin{figure}[h]
    \centering
    \includegraphics[width=0.94\linewidth]{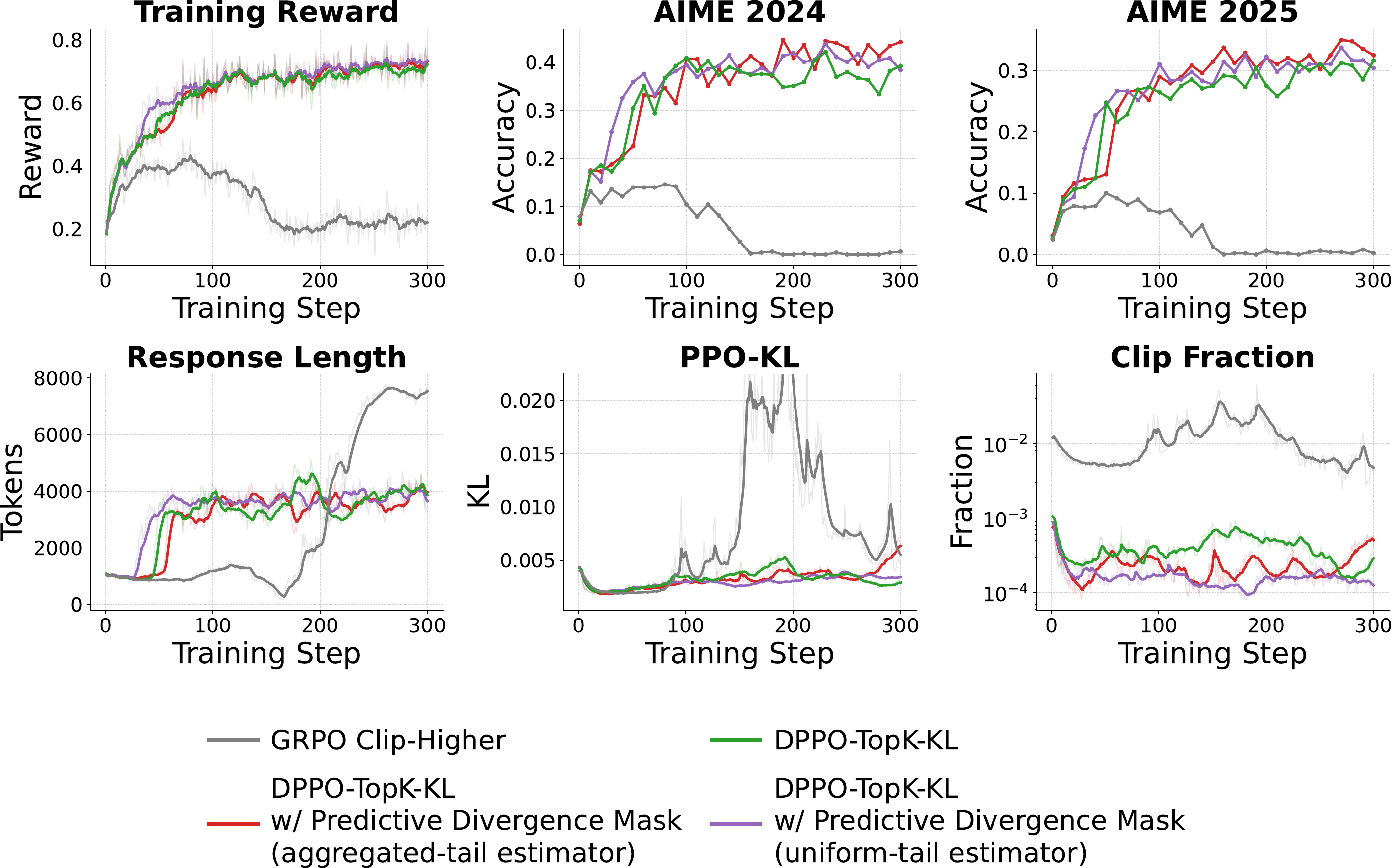}
    \caption{Training dynamics for Qwen3-30B-A3B-Base with FP8 Rollout at
    $K=20$ and $\delta=0.15$.}
    \label{fig:30ba3b_fp8rollout_full}
\end{figure}


\end{document}